\def\year{2018}\relax
\documentclass[letterpaper]{article} 
\usepackage{aaai18}  
\usepackage{times}  
\usepackage{helvet}  
\usepackage{courier}  
\usepackage{url}  
\usepackage{graphicx}  
\usepackage{hyphenat}
\usepackage{paralist,algorithm2e,amsmath,amsthm,amssymb}

\newtheorem{theorem}{Theorem}

\newtheorem{definition}[theorem]{Definition}

\newcommand{\calC}{\mathcal{C}}

\newcommand{\bbN}{\mathbb{N}}

\newcommand{\bbE}{\ensuremath{\mathbb{E}}}
\newcommand{\bbR}{\ensuremath{\mathbb{R}}}

\begin{document}
%
\title{Polynomial-time probabilistic reasoning with partial observations via 
implicit learning in probability logics}
\author{Brendan Juba\thanks{Supported by 
NSF Award CCF-1718380.}\\ Washington University in St.\ Louis\\{\tt bjuba@wustl.edu}}

\maketitle
\begin{abstract}
Standard approaches to probabilistic reasoning require that one possesses an 
explicit model of the distribution in question. But, the empirical learning of 
models of probability distributions from partial observations is a problem for 
which efficient algorithms are generally not known. In this work we consider the
use of bounded-degree fragments of the ``sum-of-squares'' logic as a probability
logic. Prior work has shown that we can decide refutability for such fragments 
in polynomial-time. We propose to use such fragments to answer queries about 
whether a given probability distribution satisfies a given system of constraints
and bounds on expected values. We show that in answering such queries, such 
constraints and bounds can be implicitly learned from partial observations in 
polynomial-time as well. It is known that this logic is capable of deriving many
bounds that are useful in probabilistic analysis. We show here that it 
furthermore captures useful polynomial-time fragments of resolution. Thus, these
fragments are also quite expressive. 
\end{abstract}

\section{Introduction}
Most scientific reasoning is probabilistic. It is quite rare for a conclusion 
to hold categorically. Informal conclusions such as ``smoking causes cancer'' 
correspond to formally probabilistic claims that the rate of incidence of
cancer is higher in a population of smokers than another in which participants
are forbidden from smoking. Likewise, our knowledge of the specific cases that
comprise a study is almost necessarily incomplete. We are often interested in
latent variables, such as whether or not a patient has a specific disease, that
we seek to infer based on some observed attributes, e.g., the manifested 
symptoms. Unfortunately, if we do not already understand the processes that
connect the observed attributes to the latent factors of interest, e.g., in the
sense of possessing a probabilistic graphical model~\cite{pearl88} of the 
system, the situation is quite challenging.

Indeed, the currently dominant approach to reasoning from data in such problems,
as embodied for example by Markov Logic Networks~\cite{rd06}, is to first
learn such a probabilistic graphical model, and then apply weighted model 
counting on translations of this model~\cite{gd11,dw12}. 
(We discuss other approaches later.) The problem lies in the 
first stage, in the ``structure learning'' problem. When the examples are not 
fully specified, existing approaches based on Expectation-Maximization (``E-M'') 
do not scale well~\cite[Section 19.4]{kf09}. Although there are a variety of 
proposals for learning, for example, models with bounded tree-width~\cite
{nb04,skb13}, the networks encountered in practice often do not feature low 
tree-width~\cite{cd08}. In a similar spirit, Poon and Domingos~\citeyear{poond11}
proposed sum-product networks as an alternative model for probability 
distributions. While sum-product networks provide some advantages over standard 
graphical models~\cite{gd13}, ultimately just as with standard graphical models, 
structure learning has been accomplished by using the E-M meta-algorithm over the
model likelihood. Thus, it falls prey to the same issues.

In this work we will consider techniques for reasoning under partial information
that bypass the structure learning problem and indeed the whole framework of
probabilistic graphical models. The structure learning problem under partial 
information remains out of reach, so what we achieve will necessarily be 
incomparable to what can be achieved given such a graphical model. Instead, we
extend probability logics, e.g., as proposed by Nilsson~\citeyear{nilsson86}
or more aptly, the generalization to expectation proposed by Halpern and 
Pucella~\citeyear{hp07}, in two ways: First, we observe that these logics that 
only consider linear inequalities can be strengthened to logics that consider 
polynomial inequalities on the support and moments. For this, we will observe 
that the ``sum-of-squares'' (a.k.a., ``Positivstellensatz refutation'') logic 
introduced by Grigoriev and Vorobjov~\citeyear{gv01} can be interpreted as a 
probability logic, following work by Lasserre~\citeyear{lasserre01} on deciding
whether or not a probability distribution can be consistent with a given set
of moments (aka, ``moment problems''). We note that this logic has a rather
powerful polynomial-time fragment: in addition to being able to derive and use a
wide variety of standard analytic inequalities, {\em we also show that these 
fragments can simulate some of the strongest fragments of resolution known to be
decidable in polynomial time}. And second, {\em we show how learning from 
partial examples can be integrated into answering queries in such a logic using 
implicit learning}~\cite{juba13-ijcai}.

\subsection{Relationship to Other Work}
The second extension will distinguish our approach both from Nilsson's logic,
which does not consider how the input probabilistic bounds are obtained, and
from approaches such as reasoning directly about the maximum entropy 
distribution that is consistent with the given constraints~\cite{bghk96}.
Meanwhile, it is of course distinguished from the work on PAC-semantics and
learning to reason~\cite{kr97,valiant00,juba13-ijcai,michael14} in that it can 
estimate general probabilities and expected values, and is not restricted to 
only drawing conclusions about what is true with high probability.

The qualitatively most similar relative to our approach is probabilistic
logic programming~\cite{drk08}, and in particular, Bayesian logic programming
from interpretations~\cite{kdr08}. In this latter problem, we suppose that
an unknown distribution is specified by some Bayesian logic
program, and we would like to synthesize it from a collection of observations
drawn from the distribution. In order for this to be feasible, it
must be that the distribution has a nice description as a Bayesian logic 
program and so on. The distinction between our approach and Bayesian logic
programming is roughly 
that in
the setting we will consider, the distribution will not be restricted, but we
will therefore necessarily give up hope of being able to generate a complete, 
compact description of the distribution (which in general may require more than
$2^{2^n}$ bits to represent even if it is supported on $n$ propositional 
attributes). Instead, we will only be able to infer some relationships between
the moments and support of the distribution, depending on how much information
the examples provide. Similarly, Khot et al.~\citeyear{knks15} proposed to learn
explicit models of the data distribution, and to use these models to perform 
inference in cases where the data has been masked at random. In addition to not
relying on the ability to represent the distribution by such models, our approach
does not assume that the data-hiding process is independent of the data 
distribution. But, our inferences will be relatively limited.

Our work relies on the close connection between fragments of the sum-of-squares 
logic and hierarchies of semidefinite program relaxations of (probability
distributions supported on) systems of polynomial inequalities pioneered
independently by Lasserre~\citeyear{lasserre01}, Parrilo~\citeyear{parrilo00}, 
Nesterov~\citeyear{nesterov00}, and Shor~\citeyear{shor87}. These works
proposed the use of these semidefinite programming hierarchies as a way of
relaxing polynomial optimization problems, which captures many discrete 
optimization problems, especially those using 0-1 integer optimization. These
techniques have been quite successful at giving new algorithms for many 
machine learning problems, including many state-of-the-art algorithms~\cite
{bks15,rs15,hss15,bm16,mss16,agmr17}. The difference is that these approaches
take the view that we are seeking to find a specific representation by solving
a semidefinite program relaxation in which the distributions are over candidate
solutions to the program. Of particular interest in this vein is work by Hopkins
and Steurer~\citeyear{hs17}, who propose this approach as a way of solving
general kinds of maximum likelihood/maximum a posteriori inference problems,
and demonstrate it on a community detection problem. Hopkins and Steurer's work
is most notable for making the connection between such an approach and Bayesian 
inference quite explicit. The main distinction here is that we are 
interested in deciding queries about the distribution under partial information,
whereas Hopkins and Steurer are interested in obtaining explicit estimates of 
latent parameters. 

We stress, however, that Lasserre~\citeyear{lasserre01} in particular viewed 
these programs as a means to decide ``moment problems,'' i.e., whether or not 
there may exist a probability distribution consistent with a given set of 
moments. This is likewise how we use these programs. The main distinction 
between our work and Lasserre's is two-fold. First, Lasserre was primarily 
interested in establishing that the hierarchy of relaxations contain some tight 
relaxation, i.e., for each system, for some sufficiently high degree, the 
program can only have actual probability distributions as its feasible 
solutions. By contrast, like the subsequent work in the algorithms community, 
we are interested in the power of the sum-of-squares programs at fixed, finite 
levels of the hierarchy for probabilistic inference. Second, we are interested 
in how to learn additional constraints by
incorporating empirical estimates of the moments from partial 
observations into the programs.

\section{The Sum-of-Squares Probability Logic}
For motivation, we will first briefly recall Nilsson's probability logic~\cite
{nilsson86}, based on linear programming, or perhaps more accurately, the
reconstruction of Nilsson's work by Fagin et al.~\citeyear{fhm90}, and the
generalization of this to reasoning about expected value by Halpern and 
Pucella~\citeyear{hp07}. 
We then recall the sum-of-squares
formulation for a system of polynomial inequalities and the connection to these
logics via the ``pseudo-expectation'' view introduced by Barak et al.~\citeyear
{brs11,bks14} (relaxing Lasserre~\citeyear{lasserre01}). 

\subsection{Probability Logics Based on Linear Programming}
In the propositional case of Nilsson's logic, we have a variable $p(\varphi)$ 
representing the likelihood of each propositional formula $\varphi$ under 
consideration. Naturally, $p(\varphi)\in [0,1]$. We have a number of constraints
that we know hold in general among such variables,such as $p(\neg\varphi)+
p(\varphi)=1$. Or, more generally, $p(\psi\wedge\varphi)+p(\neg\psi\wedge
\varphi)=p(\varphi)$, where $p(\top)=1$ and for any equivalent formulas $\psi$ 
and $\varphi$, $p(\psi)=p(\varphi)$. In Nilsson's logic, the basic relationships
between these variables are given by a system of {\em linear inequalities}, of 
the form $\sum_\varphi a(\varphi)p(\varphi)\geq c$, where the coefficients 
$a(\varphi),c\in\bbR$. We can encode all of the above relationships as linear 
inequalities, as well as relationships such as $p(\psi)\leq p(\varphi)$ if 
$\psi\models\varphi$. Now, given such a system of linear inequalities,
\begin{align*}
\sum_\varphi a^{(1)}(\varphi)p(\varphi)&\geq c^{(1)}\\
\vdots&\\
\sum_\varphi a^{(m)}(\varphi)p(\varphi)&\geq c^{(m)}
\end{align*}
for any $\alpha^{(1)},\ldots,\alpha^{(m)}\in\bbR$, the linear combination
\[
\sum_{i=1}^m\alpha^{(i)}\sum_\varphi a^{(i)}(\varphi)p(\varphi)\geq 
\sum_{i=1}^m\alpha^{(i)}c^{(i)}
\]
is also true. Nilsson's logic also allows us to infer this.

Now, given that Nilsson's logic includes propositional reasoning in its axioms, 
it is naturally NP-hard and Fagin et al.~\citeyear{fhm90} show that when we 
restrict our attention to measurable probabilities, it is in fact NP-complete. 
We will consider a weaker starting point that is tractable. All of the initial 
constraints can be written as a system of linear inequalities; suppose that some
subset of such inequalities is given. Now, given an additional linear inequality
constraint, the question of whether or not it is consistent with our initial
subset is merely linear feasibility. And, it follows from Farkas' lemma that the
system is infeasible if and only if we can derive $0\geq 1$ from the linear
inequality inference rule. So, the linear combination inference rule gives 
refutation-completeness for this (admittedly weak) logic, and algorithms for 
linear programming give a method for deciding feasibility in polynomial time. 
By binary search on lower and upper bounds for the $p(\varphi)$, we can search 
for the tightest upper and lower bounds entailed by this system by adding the 
inequalities one at a time and testing for feasibility.

Indeed, essentially the same line of reasoning carries over to the 
generalization of Nilsson's logic to reasoning about expectations by Halpern
and Pucella~\citeyear{hp07}, in which we replace the $p(\varphi)$ variables
with $e(x_i)$ variables indicating the expected value of some random variable
$x_i$. We remark that in Halpern and Pucella's logic, one can treat the original
propositional variables $\varphi$ as random variables, and recover Nilsson's 
logic. The main defect to the use of linear programming for inference is again
that we drop most of the powerful propositional reasoning capabilities of the
logic, reducing them to a handful of chosen inequalities.

\subsection{The Sum-of-Squares Positive-Semidefinite Progams}
We will adopt the perspective of Halpern and Pucella. Let us now consider, for
some fixed degree parameter $d\in\bbN$, a set of variables corresponding to the
expected values of all monomials over our random variables $x_1,\ldots,x_n$ of
total degree at most $d$, which we denote $e(x^{\vec{\alpha}})$, with the 
interpretation that $x^{\vec{\alpha}}=\prod_{i=1}^nx_i^{\alpha_i}$ where
$\sum_{i=1}^n\alpha_i\leq d$ ($\vec{\alpha}$ is thus the vector of exponents of 
the monomial). We will enforce $e(1)=1$ (for the ``empty monomial'' $1$).

We will consider both discrete, propositional variables and bounded,
continuous-valued variables.
We will use the following standard set of constraints for propositional
variables: for each propositional variable, we will include variables 
$x_i$ and $\bar{x}_i$ encoding the variable $x_i$ and its negation, related by
the {\em complementarity axiom}, $x_i-\bar{x}_i+1=0$. Each will also be 
constrained by the {\em Boolean axiom}, $x_i^2-x_i=0$ (and $\bar{x}_i^2-
\bar{x}_i=0$). For other, continuous variables $x_i$, we will usually assume
that there is an upper and lower bound on the range $B_i,L_i\in\bbR$, and we 
will include explicit bounding inequalities: $\max\{B_i^2,L_i^2\}-x_i^2\geq 0$
for each $x_i$, and for each monomial $x^{\vec{\alpha}}$ in which $\alpha_i\leq
1$ for all $i$, we give an upper and lower bound, $B_{\vec{\alpha}}$ and 
$L_{\vec{\alpha}}$, on the monomial as follows: we compute the largest and 
smallest magnitude positive and negative values consistent with the bounds on 
the range of each variable (if any) in the monomial. These may be computed in 
linear time by a simple dynamic programming algorithm that iteratively considers
only the bounds for the product of the first $k$ attributes. We finally take 
the largest of these values for $B_{\vec{\alpha}}$ and the smallest for 
$L_{\vec{\alpha}}$. We then include $B_{\vec{\alpha}}-x^{\vec{\alpha}}\geq 0$, 
and $x^{\vec{\alpha}}-L_{\vec{\alpha}}\geq 0$; bounds on all monomials will
follow from these.\footnote{%
We need to separately bound the degree-1 moments since, as we will see,
inequality constraints can only be multiplied by polynomials of even degree. We 
note that from the Boolean axioms we will be able to easily derive $x_i\geq 0$ 
and $x_i\leq 1$, so we need not include additional inequalities for these 
variables.}  
Any additional constraints on the distribution's support in the 
background knowledge and query system will be added to the system of defining 
polynomial constraints.

Now, consider the matrix in which rows and columns are indexed by the monomials
of degree up to $d/2$ in some standard way, and the $(\vec{\alpha},\vec{\beta})$
entry of this matrix is the variable $e(x^{\vec{\alpha}+\vec{\beta}})$. That is,
for the vector $\vec{v}$ of variables indexed in the same order, the matrix is
$\bbE[\vec{v}\vec{v}^\top]$, which is certainly positive semidefinite. We refer 
to this as a {\em moment matrix}. Using a semidefinite program, we can 
strengthen our original linear program formulation by adding the constraint that
this moment matrix must be positive semidefinite.

We can interpret this as follows. We note that for any real vector $\vec{p}$ 
again indexed by the monomials, $\vec{p}^\top\vec{v}$ gives the expected value 
of the polynomial with coefficients given by $\vec{p}$: $\sum_{\vec{\alpha}}
p_{\vec{\alpha}}e(x^{\vec{\alpha}})=\bbE[\sum_{\vec{\alpha}}p_{\vec{\alpha}}
x^{\vec{\alpha}}]$. So, the positive semidefiniteness of this moment matrix
corresponds to requiring that in any feasible solution, the square of any 
polynomial has a nonnegative expected value.

Another family of constraints that we will use is the following. Suppose
that we know the polynomial constraint $g(\vec{x})\geq 0$ holds for all $\vec{x}$ in
the support of the distribution. Then surely, $\bbE[p(\vec{x})^2g(\vec{x})]\geq
0$ for any polynomial $p(\vec{x})$. We can capture such a constraint as follows:
again, we consider a matrix in which the rows and columns are indexed by the 
monomials $\vec{\alpha}$ up to degree $(d-d')/2$ where $g(\vec{x})$ has total
degree $d'$, such that in position $(\vec{\alpha},\vec{\beta})$ we have
$\sum_{\vec{\gamma}}g_{\vec{\gamma}}e(x^{\vec{\alpha}+\vec{\beta}+
\vec{\gamma}})$. If we then assert that this {\em ``localizing matrix''} for
$g$ is positive semidefinite, it indeed imposes the desired constraint on the
possible values for the $e(x^{\vec{\alpha}})$.

Finally, similarly, if $h(\vec{x})=0$ holds for all $\vec{x}$ in the distribution's
support, then $\bbE[p(\vec{x})h(\vec{x})]=0$ for all polynomials $p$. We can
add a linear constraint $\sum_{\vec{\gamma}}h_{\gamma}x^{\vec{\alpha}+
\vec{\gamma}}=0$ for all $x^{\vec{\alpha}}$ of degree at most $d-d'$ when $h$
has total degree $d'$.

The resulting semidefinite program given by this system of moment matrix and
localizing matrix constraints is referred to as the {\em degree-$d$ 
sum-of-squares relaxation} of the system of polynomial inequalities~\cite
{shor87,nesterov00,parrilo00,lasserre01}. We can test the feasibility
of this system in polynomial time using semidefinite program
solvers. As we will recall next, analogous to our earlier consequence of
Farkas' Lemma, the feasibility of the resulting semidefinite programs is 
captured by a simple algebraic logic under some general conditions---in
particular given we included explicit bounds on the range of the random 
variables among the defining polynomial inequalities.

\subsection{Sum-of-Squares Refutations for Probabilities}
For polynomials $g_1,\ldots,g_r$ and $h_1,\ldots,h_s$ in
$\bbR[x_1,\ldots,x_n]$, consider the set of $\vec{x}\in\bbR^n$ satisfying
$g_j(\vec{x})\geq 0$ (for $j=1,\ldots,r$) and $h_k(\vec{x})=0$ (for
$k=1,\ldots,s$). A {\em sum-of-squares} is, as the name suggests, a polynomial
$\sigma(\vec{x})=\sum_{\ell=1}^tq_\ell(\vec{x})^2$ for polynomials $q_1,\ldots,q_t\in
\bbR[x_1,\ldots,x_n]$. It is easy to see that a sum-of-squares $\sigma$ must be
non-negative for every $\vec{x}\in \bbR^n$. Thus,
if we can find sums-of-squares $\sigma_0,\sigma_1,\ldots,\sigma_r$ and polynomials
$u_1,\ldots,u_s$ such that
\begin{equation}\label{refutation}
\sigma_0(\vec{x})+\sum_{j=1}^r\sigma_j(\vec{x})g_j(\vec{x})+\sum_{k=1}^su_k(\vec{x})h_k(\vec{x}) = -1
\end{equation}
the set defined by the inequalities $g_1(\vec{x})\geq 0,\ldots,g_r(\vec{x})\geq
0$ and $h_1(\vec{x})=0,\ldots,h_s(\vec{x})=0$ must be empty (or else we would
reach a contradiction).
The {\em sum-of-squares} proof system of Grigoriev and Vorobjov~\citeyear{gv01}
is to show that such systems are unsatisfiable by finding such polynomials:
\begin{definition}[Sum-of-squares refutation: syntax]
A {\em sum-of-squares refutation} of a system of equalities
$h_1(\vec{x})=0,\ldots,h_s(\vec{x})=0$ and inequalities
$g_1(\vec{x})\geq 0,\ldots,g_r(\vec{x})\geq 0$ consists of
$\sigma_0,\ldots,\sigma_r$ and $u_1,\ldots,u_s$ satisfying Equation~\ref{refutation}. The
{\em degree} of the sum-of-squares refutation is degree of the resulting formal
expression, while the {\em size} is the number of monomials in $\sigma_0,\ldots,\sigma_r$
and $u_1,\ldots,u_s$.
\end{definition}
\noindent
Now, noting that the program capturing the existence of a sum-of-squares 
refutation is the dual to the sum-of-squares relaxation, one obtains
\begin{theorem}[
Soundness \cite{shor87,nesterov00,parrilo00,lasserre01}]
\label{ineq-dual}
Let $g_1(\vec{x})\geq 0,\ldots,g_r(\vec{x})\geq 0$, $h_1(\vec{x})=0,\ldots,
h_s(\vec{x})=0$ be a system of constraints
that is explicitly compact.
Then either there is a degree-$d$ sum-of-squares refutation or there is a
solution to the degree-$d$ sum-of-squares relaxation.
\end{theorem}

Note that this means that if there is a sum-of-squares refutation, then there 
cannot be any probability distributions consistent with the given support 
constraints, or else these would give a feasible solution to the semidefinite 
program we developed in the previous section. Thus, for any degree, 
sum-of-squares is a {\em sound}, but possibly incomplete refutation system for 
probability distributions, which can be decided in polynomial time (in the 
number of attributes and size of the coefficients in the proofs) by solving the 
corresponding semidefinite program. 

We can furthermore add prior knowledge of constraints on the moments of the 
distribution to the semidefinite program, and read off an extended kind of 
sum-of-squares proof from the dual: If we add the constraint 
$e(x^{\vec{\alpha}})\leq\gamma$ to the semidefinite program, this corresponds to
adding nonnegative multiples of the polynomial $(\gamma-x^{\vec{\alpha}})$. A 
sum-of-squares derivation of $-1$ yields a contradiction now by considering the 
expected value of the derived expression for a distribution additionally 
satisfying the moment constraints: the expression is equal to $-1$, but 
$\bbE[\gamma-x^\alpha]\geq 0$ by definition and the rest of the expression is 
again nonnegative on the distribution's support. Indeed, for any polynomial
$p(\vec{x})$, we can incorporate knowledge that $\bbE[p(\vec{x})]\leq \gamma$
by adding nonnegative multiples of $(\gamma - p(\vec{x}))$ in the dual.

If the set is {\em ``explicitly compact,''} e.g., if the variables are either
all explicitly Boolean (the constraint $x_i^2-x_i=0$ is included) or bounded
($x_i^2\leq B_i$ is included for some $B_i\geq 0$) Putinar's 
Positivstellensatz~\cite{putinar} asserts that such proofs {\em must} exist,
i.e., the system is {\em complete}.
\begin{theorem}[Completeness, a corollary of Putinar~\citeyear{putinar}]\label{putinar-sdp}
There exists a probability distribution with expected values 
$\{e(x^{\vec{\alpha}})\}_{\vec{\alpha}\in\bbN^n}$ supported on a
set given by an explicitly compact system $g_1(\vec{x})\geq 0,\ldots,
g_r(\vec{x})\geq 0$, $h_1(\vec{x})=0,\ldots,h_s(\vec{x})=0$
iff every moment matrix is positive semidefinite, every localizing matrix for 
each $g_j$ is postive semidefininte, and every localizing matrix for each $h_k$
is zero (i.e., solutions exist for all degrees $d$).
\end{theorem}

\section{The Expressive Strength of Sum-of-Squares}
We now briefly review the surprising strength of these proofs, and establish 
some new results on their Boolean reasoning ability.
As with the logics of Nilsson~\citeyear{nilsson86} and Halpern and 
Pucella~\citeyear{hp07}, by searching over feasible inequality constraints, we 
can use sum-of-squares queries to compute upper and lower bounds on 
probabilities (of Boolean expressions) and more generally expectations
of polynomials of real-valued variables, thus bounding the variance
and so on. We can estimate conditional probabilities $\Pr[x_a|x_b]$ by 
separately estimating bounds on $\Pr[x_a\wedge x_b]=\bbE[x_ax_b]$ and 
$\Pr[x_b]=\bbE[x_b]$. The key question is how tight these bounds are.

Much of the excitement about sum-of-squares in algorithms stems from the power
of even low-degree (e.g., degree-4) sum-of-squares proofs to capture a 
surprisingly strong fragment of probabilistic reasoning. In particular, Barak et
al.~\citeyear{bbhksz12} show that such sum-of-squares proofs exist for key 
analytic inequalities such as the Cauchy-Schwarz and H\"{o}lder inequalities and
hypercontractivity bounds. O'Donnell and Zhou~\citeyear{oz13} and Kauers et 
al.~\citeyear{kotz14} further established that low-degree sum-of-squares can 
derive reverse hypercontractivity bounds, and thus also the KKL Theorem~\cite
{kkl88} and an invariance principle (generalizing the central limit 
theorem)~\cite{moo10}. These theorems were used previously to analyze instances
of optimization problems that were hard for existing algorithms. These works
then observed that the sum-of-squares relaxation can actually then derive tight
bounds that solve these very same instances since sum-of-squares refutations
capture all of these constraints. 
We thus observe that sum-of-squares is quite powerful for probabilistic
analysis.

For Boolean reasoning, sum-of-squares simulates a variety of fragments of systems 
for which polynomial-time algorithms were already known. 
Notably, Berkholz~\citeyear{berkholz18} has shown that degree-$2d$ sum-of-squares
simulates the degree-$d$ fragment of Polynomial Calculus when we include the
Boolean axioms $x_i^2-x_i=0$ for all variables $x_i$. Polynomial calculus is a
system that allows us to derive new polynomial equations by taking linear
combinations of previous lines, or by multiplying a previous line by an 
indeterminate (variable). In particular, if we include a variable for each literal
and the complementarity axiom $x_i+\bar{x}_i-1=0$, the degree-$d$
fragment easily simulates {\em width-$d$} resolution~\cite{absrw02}, where recall
the width is the maximum number of literals in any clause. This is accomplished
by representing the clauses as a system of single monomial constraints of degree 
at most $d$: for the clause $\ell_1\vee\cdots\vee\ell_k$, we have $\bar{\ell}_1
\cdots \bar{\ell}_k=0$ is an equivalent constraint of degree at most $d$. 

We further show that sum-of-squares simulates space-bounded treelike resolution 
without the bounded-width requirement. Although the encoding used in the above 
works cannot express wide clausal constraints on the support, we can circumvent 
this by using a constraint $\sum_i\ell_i\geq 1$ on the support to encode the 
clause $\bigvee_i\ell_i$ (with each $\ell_i$ also constrained by $\ell_i^2-
\ell_i=0$). We will use the following 
characterization of space-bounded treelike resolution given by Ans\'{o}tegui et 
al.~\citeyear{ablm08}: For a {\em partial assignment} to propositional variables
$x_1,\ldots,x_n$, $\rho\in\{0,1,*\}^n$ (where $*$ means ``unassigned'') the 
partial evaluation or restriction of a formula $\varphi$ by $\rho$, denoted 
$\varphi|_\rho$, is obtained by plugging in $\rho_i$ for $x_i$ when $\rho_i\neq 
*$, and simplifying the Boolean connectives in the natural way: we delete true 
inputs to ANDs, false inputs to ORs, and replace ANDs and ORs respectively with 
false and true if they have, respectively, a false or true input. (Negations are
simply evaluated if their input is.) Recall that {\em unit propagation} on a set
of clauses $\calC$ and a partial assignment $\rho$ is the inference rule that, 
if some clause $C\in\calC$ simplifies to a single literal $\ell$ under 
partial evaluation by $\rho$ then we can infer the setting that satisfies $\ell$
and add it to $\rho$. If we ever obtain an empty clause, $C|_\rho=\bot$, then 
this is a {\em unit propagation refutation}. Ans\'{o}tegui et al.~\citeyear
{ablm08} essentially observe that unit propagation captures ``clause space 1'' 
treelike resolution refutations, and can be generalized to space $s$ by allowing
$s-1$ levels of the following {\em failed literal} rule: inductively, supposing 
we have defined the system with $s-1$ levels (where $0$ levels is unit 
propagation), $s$ levels of the failed literal rule is as follows: if, when we 
add an assignment $x_i=b$ to our partial assignment $\rho$, we can obtain a 
level-$s-1$ refutation, then we can infer $x_i=1-b$, and add this to $\rho$ 
instead. A level $s$ refutation then occurs if we can eventually obtain a 
level-$s-1$ refutation without asserting any additional literals (thus, 
eventually, a unit propagation refutation). Ans\'{o}tegui et al.~found that 
instances of resolution that are easy in practice are refutable with relatively 
small $s$. Note that unit propagation alone ($s=1$) already simulates chaining,
e.g.~in Horn KBs.

\begin{theorem}
The degree-$s+1$ sum-of-squares relaxation of the linear inequality encoding of 
a CNF is infeasible if there exists a level-$s$ refutation.
\end{theorem}
\begin{proof}
Naturally, we proceed by induction on $s$. First, for level-$0$, we argue that
degree-$1$ sum-of-squares detects unit propagation refutations among the input
system. We will argue this by induction on the number of steps of unit 
propagation inference. In $0$ steps, if there is an empty clause in the input,
we have the constraint $-1\geq 0$ in the input system, which is a trivial
contradiction. Now, after $i$ steps of inference, suppose we have $\ell_i\leq 
0$ (or $\ell_i=0$) for $i\neq i^*$ in some clause. Then we can derive $\ell_{i^*}
\geq 1$ in degree 1 (and hence, $\bar{\ell}_{i^*}\leq 0$) from $\sum_i\ell_i-1+
\sum_{i\neq i^*}-\ell_i=\ell_{i^*}-1$. Observe that these linear combinations do
not increase the degree. Thus, finally, if we have derived $\ell_i\leq 0$ for
all literals in some clause, we can similarly derive $-1\geq 0$ to obtain our
contradiction in degree-$1$.

Now, given that the program detects level-$s-1$ refutations in degree $s$,
we argue that it detects level $s$ refutations in degree $s+1$: indeed, we will
argue that if $\bar{\ell}$ can be inferred by an application of the failed
literal rule and if there is a feasible solution to the degree $s+1$ 
sum-of-squares relaxation, that for all monomials $x^{\alpha}$, $e(\ell 
x^{\vec{\alpha}})=0$. Indeed, suppose not. Consider the Cholesky factorization of 
the (positive semidefinite) moment matrix into $UU^\top$. We know that we must 
have that for the row indexed by $\ell$, $\vec{u}_\ell\neq 0$, so $e(\ell)=
e(\ell^2)>0$. It follows then that we can perform the following 
``{\em conditioning operation}'' (from Karlin et al.~\citeyear{kmn11}): for the 
minor of the moment and localizing matrices consisting of all index monomials 
with a factor of $\ell$, we know that this is a positive semidefinite minor. If 
we rescale each by $1/e(\ell^2)$, then indeed it further satisfies the constraint 
that the entry we obtained from the $(\ell,\ell)$ position takes value $1$.  We 
thus see that this is a solution to the sum-of-squares program of degree $s$ in 
which moreover it may be shown that by the localizing constraints for the Boolean 
axiom on $\ell$, for any monomial $x^{\vec{\alpha}}$, $e(\ell x^{\vec{\alpha}})=
e(x^{\vec{\alpha}})$. More generally by the localizing constraints on the 
complementarity axiom, $e(\bar{\ell})=0$ in this solution, as is
any monomial containing $\bar{\ell}$. Thus, we obtain a solution in which $\ell
=1$ and $\bar{\ell}=0$. But since we are supposing that we can infer 
$\bar{\ell}$ by an application of the level-$s$ failed literal rule,  this
means that there is a level-$s-1$ refutation of this system. By IH, therefore,
this system must be infeasible, a contradiction. It must be then that indeed
$e(\ell x^{\vec{\alpha}})=0$ for all monomials $x^{\vec{\alpha}}$ as claimed
and therefore also $e(\bar{\ell}x^{\vec{\alpha}})=e(x^{\vec{\alpha}})$; in
particular, then, $e(\bar{\ell})=e(1)=1$ in any feasible solution.
Finally if some input clause simplifies to the empty clause, we again obtain that
$-1\cdot e(1)\geq 0$ in contradiction to our constraint that $e(1)=1$. Thus, if 
there is a level-$s$ refutation, the degree $s+1$ sum-of-squares relaxation must 
be infeasible.
\end{proof}

\paragraph{Weaknesses.} So we see how we can use sum-of-squares to compute
estimates on probabilities and bounds on expected values, using the complexity
of deriving the bound rather than the form of the distribution as our ``bias.''
Compared to, e.g., graphical models, this approach circumvents some fundamental
difficulties, but has some other inherent weaknesses.
The main weakness is that we cannot express prior
knowledge about the independence of random variables in our distribution.
Indeed, this would lead to quadratic (or worse) constraints in the ``primal''
optimization problem, preventing us from solving it in polynomial time.
A second weakness is that in our tractable fragments, we can only infer values
for marginal distributions that refer to at most $O(1)$ variables at a time.
(That is, we must ``sum out'' all but $O(1)$ variables.) We may have constraints
on more variables as long as these are expressed by a low-degree polynomial,
such as our linear inequality encodings of clauses, but this only captures a
limited family of the possible constraints.

\section{Implicit Learning From Partial Examples}

We now suppose that we have access to {\em partial examples} drawn from the
distribution $D$ we are seeking to reason about. We would like to use these
examples to empirically learn additional constraints that the distribution
satisfies. We will develop a notion of implicit learning of constraints for
use in answering queries in our probability logic, analogous to Juba~\citeyear
{juba13-ijcai} for Boolean logics.

\subsection{Masking and Testability}
We will use a model of learning under partial information adapted from
prior works~\cite{michael10,rubin76}. In such models, the distribution over 
partial examples is produced as follows. First, a ``ground truth'' example 
$\vec{x}$ is drawn from the distribution $D$. Then there is a second random 
process $M$ that independently selects a function $m$ that takes complete
examples $\vec{x}$ and produces a {\em partial example} $\vec{\rho}$ by 
replacing any number of coordinates of $\vec{x}$ with the value $*$ meaning
``unknown'' or ``missing.'' In particular, the choice of which coordinates to
hide may depend on all of the actual values of $\vec{x}$. $M$, like $D$, is
arbitrary in general. Indeed, $m(\vec{x})$ may even hide the entire example at 
will. Thus, we can only hope to learn constraints from the partial examples 
$\vec{\rho}$ it produces that are revealed by $M$ to be obeyed with high 
probability.

A family of such constraints is as follows. Again, we suppose that each $i$th
attribute is known to have some upper and lower bound $B_i$ and $L_i$, and that
the corresponding constraints are included in the program. For a given partial 
example $\vec{\rho}$, there also exists an upper and lower bound, $B_{\vec
{\alpha}}$ and $L_{\vec{\alpha}}$, on the monomial $x^{\vec{\alpha}}$ given 
$\vec{\rho}$ as follows: first, 
for each $x_i^{\alpha_i}$, if $\rho_i\neq *$, then we substitute
 $\rho_i^{\alpha_i}$ for $\rho_i$. Then we compute the largest and smallest 
magnitude positive and negative values consistent with the bounds on the range 
of each variable (if any) in the monomial as before.

\begin{definition}
We will say that a polynomial inequality 
$\sum_{\vec{\alpha}}p_{\vec{\alpha}}x^{\vec{\alpha}}\geq 0$ is {\em witnessed} 
under $\vec{\rho}$ if, when we plug in the upper bound for each 
$x^{\vec{\alpha}}$ under $\vec{\rho}$ for every $p_{\vec{\alpha}}< 0$ and the
lower bound when $p_{\vec{\alpha}} >0$, the inequality is still
satisfied. We will say that a system of polynomial constraints 
$p_1(\vec{x})\geq 0,\ldots,p_\ell(\vec{x})\geq 0$ is
{\em $(1-\gamma)$-testable} with respect to $M$ and $D$ if with probability
at least $(1-\gamma)$ over $\vec{\rho}$ drawn from $M$ and $D$, every 
$p_i(\vec{x})\geq 0$ in the system is witnessed under $\vec{\rho}$. 
\end{definition}
We will
show that knowledge encoded by all such testable systems of constraints can be
efficiently learned implicitly, even though there may be many such constraints,
as long as $\gamma$ is small enough.
In support of the empirical estimation of such expressions, we will need the 
following two definitions.

\begin{definition}
The {\em upper bound} of a polynomial expression 
$\sum_{\vec{\alpha}}p_{\vec{\alpha}}x^{\vec{\alpha}}$ is given by substituting 
the upper bound on $x^{\vec{\alpha}}$ (for the empty partial assignment) for 
each monomial $x^{\vec{\alpha}}$ if $p_{\vec{\alpha}}>0$, and substituting the lower 
bound on $x^{\vec{\alpha}}$ for $p_{\vec{\alpha}}<0$. The {\em lower bound} is similarly
given by substituting lower bounds on $x^{\vec{\alpha}}$ for $p_{\vec{\alpha}}>0$ and
upper bounds on $x^{\vec{\alpha}}$ for $p_{\vec{\alpha}}<0$. 
The {\em na\"{\i}ve norm} of the polynomial expression is then given by the
maximum of its upper bound and the absolute value of its lower bound.
\end{definition}

\begin{definition}
For a polynomial $p(\vec{x})$, we will let 
$p|_{\vec{\rho}}(\vec{x})$ denote the polynomial with $\rho_i$ plugged in for
$x_i$ whenever $\rho_i\neq *$, and collecting terms with the same monomial. We 
refer to this as the {\em partial evaluation} of $p$ under $\vec{\rho}$.
\end{definition}

Note that both can be computed in linear time in the size of the expression, 
given our bounds on the individual monomials and the partial assignment,
respectively.

\subsection{Deciding Queries With Implicit Learning}

Given a degree bound $d$, a collection of partial examples $\vec{\rho}^{(1)},
\ldots,\vec{\rho}^{(m)}$, a system of support constraints,
$g_j(\vec{x})\geq 0$ (for $j=1,\ldots,r$, that we assume includes the upper and
lower bound constraints for each attribute) and $h_k(\vec{x})=0$ (for
$k=1,\ldots,s$), and a system of moment constraints,
$p_\ell(\vec{x})\geq 0$ (for $\ell=1,\ldots,t$), we write the following
semidefinite program. For each $i=1,\ldots,m$, we create a set of variables
$\vec{x}^{(i)}$ for those variables unfixed in $\vec{\rho}^{(i)}$. We add a
degree-$d$ moment matrix constraint for $\vec{x}^{(i)}$, a degree-$d$
localizing matrix constraint for each $g_j|_{\vec{\rho}^{(i)}}(\vec{x}^{(i)})$,
and a degree-$d$ localizing matrix constraint for each 
$h_k|_{\vec{\rho}^{(i)}}(\vec{x}^{(i)})$.\footnote{%
We implicitly add the constraint that for $i\neq i'$, any monomial containing 
variables from the variables for examples $i$ and $i'$ gets value 0. But since
the moment matrix then has a block-diagonal structure, it is equivalent to
simply impose the constraint that each $i$th block has a positive semidefinite
moment matrix.}
Finally, we write for each monomial $x^{\vec{\alpha}}$ of degree up to $d$,
the constraints 
\begin{align*}
\frac{1}{m}\sum_{i=1}^m (x^{\vec{\alpha}})|_{\vec{\rho}^{(i)}}-
\frac{L_{\vec{\alpha}}\sqrt{\ln({2n\choose\leq d}/\delta)}}{\sqrt{2m}} 
\leq x^{\vec{\alpha}}&\\
x^{\vec{\alpha}}
\leq \frac{1}{m}\sum_{i=1}^m (x^{\vec{\alpha}})|_{\vec{\rho}^{(i)}}+\frac{B_{\vec{\alpha}}\sqrt{\ln({2n\choose \leq d}/\delta)}}{\sqrt{2m}}&
\end{align*}
and add the constraints $p_\ell(\vec{x})\geq 0$. We then accept the system
if and only if the semidefinite program is feasible.

We will need Hoeffding's inequality for the analysis, to establish that these
empirical constraints are valid:

\begin{theorem}[Hoeffding's inequality]
Let $X_1,\ldots,X_m$ be i.i.d.~$[0,1]$-valued random variables.
Let $\bar{X}=\frac{1}{m}\sum_{i=1}^mX_i$. Then for $\gamma>0$,
$\Pr[\bar{X}-\bbE[X_i] > \gamma]\leq e^{-2m\gamma^2}$.
\end{theorem}

\begin{theorem}
For $m=\Omega(S^2(d\log n+\log\frac{1}{\delta}))$ where for all
$\vec{\alpha}$ of degree at most $d$, $B_{\vec{\alpha}}-L_{\vec{\alpha}}\leq 
S$, with probability $1-\delta$:
\begin{compactitem}
\item if $D$ satisfies the given system of constraints the algorithm accepts
\item if there is a system of constraints that is $(1-\frac{1}{2S})$-testable 
under $M$ and $D$ and, together with the given system, completes a 
sum-of-squares refutation with na\"{\i}ve norm at most $S$, then the algorithm 
rejects.
\end{compactitem}
\end{theorem}
\noindent
Note that for a Boolean system, the na\"{\i}ve norm closely corresponds to the
size of the proof (in number of monomials).

\begin{proof}
We first analyze the first case. Since the distribution is assumed to satisfy
the support constraints, for any $\vec{\rho}=m(\vec{x})$ drawn from $M$ and $D$,
since $\vec{x}$ is in the support of $D$, the completion of $\vec{\rho}$ to
$\vec{x}$ is an assignment that satisfies all of the support constraints.
Moreover, since for each moment, $x^{\vec{\alpha}}$ has $B_{\vec{\alpha}}-
L_{\vec{\alpha}}\leq S$, we see that $\frac{1}{S}(x^{\vec{\alpha}}-
L_{\vec{\alpha}})$ is a random variable in the range $[0,1]$;
therefore, by Hoeffding's inequality, for our choice of $m$ the moment 
constraints hold with probability $1-\delta/{2n\choose \leq d}$ for each
$\vec{\alpha}$. So, by a union bound over these moments, with probability
$1-\delta$, $\frac{1}{m}\sum_{i=1}^m(x^{(i)})^{\vec{\alpha}}$ is feasible in
these constraints, as therefore is $\frac{1}{m}\sum_{i=1}^m
(x^{\vec{\alpha}})|_{\vec{\rho}^{(i)}}$. So our empirical upper and lower bound
constraints on $x^{\vec{\alpha}}$ are satisfied. We also know that 
$\bbE[x^{\vec{\alpha}}]$, by assumption, satisfies the given polynomial
constraints on our moments. Thus the sum-of-squares relaxation is feasible and
so our algorithm accepts with probability $1-\delta$ in this case as needed.

Now, for the second case. We assume that there is a set of $(1-1/2S)$-testable
support constraints, $\tilde{g}_{\tilde{j}}(\vec{x})\geq 0$ for $\tilde{j}=1,
\ldots,\tilde{r}$ and $\tilde{h}_{\tilde{k}}(\vec{x})=0$ for $\tilde{k}=1,
\ldots,\tilde{s}$, and moment constraints, $\tilde{p}_{\tilde{\ell}}(\vec{x})
\geq 0$ for $\tilde{\ell}=1,\ldots,\tilde{t}$, that complete a degree-$d$ 
sum-of-squares refutation with na\"{\i}ve norm at most $S$. Suppose we 
substitute $x^{\vec{\alpha}}|_{\vec{\rho}^{(i)}}$ for $x^{\vec{\alpha}}$ for 
each $\vec{\alpha}$ in this proof. Observe that it remains formally identical to
$-1$, the degree does not increase, and the sum-of-squares expressions remain
sum-of-squares expressions (on $\vec{x}^{(i)}$). 

Now, suppose we average over these $m$ refutations. Again, we obtain an 
expression formally identical to $-1$ of the same degree. The final 
sum-of-squares term $\sigma_0(\vec{x})$ becomes another sum-of-squares, 
$\frac{1}{m}\sum_{i=1}^m\sigma_0(\vec{x})|_{\vec{\rho}^{(i)}}$, of the same 
degree, so we can find an analogous expression in our program. For portions of 
the expression derived from the support constraints, for example the term 
$\sigma_j(\vec{x})g_j(\vec{x})$, we have $\sum_{i=1}^m(\frac{1}{m}
\sigma_j(\vec{x})g_j(\vec{x}))|_{\vec{\rho}^{(i)}}$. So, we see that this can 
still be written using the support constraints for the individual examples 
included in our program. The equality support constraints can be similarly 
transformed.

Next, we consider the moment bounds, which we know can only be multiplied by
positive constants. We find that the expressions $\lambda_\ell p_\ell(\vec{x})$,
i.e., of the form $\lambda_\ell\sum_{\vec{\alpha}}p_{\vec{\alpha}}x^{\vec{\alpha}}$. We see again that under this substitution we obtain an expression that may
be rewritten as $\sum_{i=1}^m\frac{\lambda_\ell}{m}\sum_{\vec{\alpha}}p_{\vec{\alpha}}x^{\vec{\alpha}}|_{\vec{\rho}^{(i)}}$. Now, our program provides us
with the moment constraint $p_\ell(\vec{x})$ and the upper and lower empirical
bounds on each moment, 
$\frac{1}{m}\sum_{i=1}^mx^{\vec{\alpha}}|_{\vec{\rho}^{(i)}}+
\frac{B_{\vec{\alpha}}\sqrt{\ln({2n\choose\leq d}/\delta)}}{\sqrt{2m}}-
x^{\vec{\alpha}}\geq 0$ and 
$\frac{L_{\vec{\alpha}}\sqrt{\ln({2n\choose\leq d}/\delta)}}{\sqrt{2m}}-
\frac{1}{m}\sum_{i=1}^mx^{\vec{\alpha}}|_{\vec{\rho}^{(i)}}+x^{\vec{\alpha}}
\geq 0$. By choosing the appropriate bound (the first if $p_{\vec{\alpha}}>0$ 
and the second if $p_{\vec{\alpha}}<0$), multiplying by $|p_{\vec{\alpha}}|>0$ 
in each case, and summing these up, we can obtain an expression of the form 
\[
\frac{1}{m}\sum_{i=1}^mp_{\ell}(\vec{x})|_{\vec{\rho}^{(i)}}-p_\ell(\vec{x})+\frac{\sum_{\vec{\alpha}}|p_{\ell,\vec{\alpha}}C_{\vec{\alpha}}|\sqrt{\ln\frac{{2n\choose\leq d}}{\delta}}}{\sqrt{2m}}
\]
where $C_{\vec{\alpha}}$  
is at most our bound on $x^{\vec{\alpha}}$.
We can add a $\lambda_\ell$-multiple of this to $\lambda_\ell p_\ell(\vec{x})$ 
to obtain $\frac{1}{m}\sum_{i=1}^m\lambda_\ell p_\ell(\vec{x})|_{\vec
{\rho}^{(i)}}$ as in the averaged refutation, plus a bounded 
error term.

Finally we consider the testable constraints in the expression. Note that 
without the terms from these constraints, the original sum-of-squares refutation
sums to
\[
-1 -\left(
\sum_{\tilde{j}=1}^{\tilde{r}}\tilde{\sigma}_{\tilde{j}}(\vec{x})\tilde{g}_{\tilde{j}}(\vec{x})+
\sum_{\tilde{k}=1}^{\tilde{s}}\tilde{u}_{\tilde{k}}(\vec{x})\tilde{h}_{\tilde{k}}(\vec{x})+
\sum_{\tilde{\ell}=1}^{\tilde{t}}\lambda_{\tilde{\ell}}\tilde{p}_{\tilde{\ell}}(\vec{x})
\right)
\]
where of course, the $\tilde{\sigma}_{\tilde{j}}(\vec{x})$ are sums of squares and the
$\tilde{u}_{\tilde{k}}(\vec{x})$ are arbitrary polynomials. Now, for 
our choice of $m$, with probability $1-\delta$, for all but $\frac{2}{3S}m$ of
the partial examples, our constraints are simultaneously witnessed under $\vec
{\rho}^{(i)}$. By the definition of witnessing, the lower bounds under each such
$\vec{\rho}^{(i)}$ on the corresponding witnessed terms in the assumed 
sum-of-squares refutation are at least $0$, and these can be derived from our
support bounds for each $\vec{\rho}^{(i)}$. Thus, using our bounds on monomials,
we can bound the averaged formal expression by
\begin{align*}
-1 &-\frac{1}{m}\sum_{\substack{\vec{\rho}^{(i)}\text{ not}\\\text{witnessed}}}
\left(
\sum_{\tilde{j}=1}^{\tilde{r}}(\tilde{\sigma}_{\tilde{j}}(\vec{x})\tilde{g}_{\tilde{j}}(\vec{x}))|_{\vec{\rho}^{(i)}}+\right.\\
&\left.
\sum_{\tilde{k}=1}^{\tilde{t}}(\tilde{u}_{\tilde{k}}(\vec{x})\tilde{h}_{\tilde{k}}(\vec{x}))|_{\vec{\rho}^{(i)}}+
\sum_{\tilde{\ell}=1}^{\tilde{t}}\lambda_{\tilde{\ell}}\tilde{p}_{\tilde{\ell}}(\vec{x})|_{\vec{\rho}^{(i)}}
\right)\leq -1+\frac{2}{3}
\end{align*}
since the testable portion of the proof has lower bound at least $-S$. Note,
moreover, that we can derive the upper bound on the unwitnessed expressions
using our bounds on the empirical monomials, obtaining some additional error
terms. Summing over all of the error terms, we find that since the overall proof
has na\"{\i}ve norm at most $S$, we can obtain that the total error term is at 
most $1/6$ by an appropriate choice of $m$. Thus, overall, we see that we can
construct a sum-of-squares derivation of $-1/6$, and hence also of $-1$ since
we can multiply each coefficient by $6$ and still obtain a valid sum-of-squares
expression for the system. Therefore we detect that the system is infeasible and
reject with probability $1-\delta$ as needed in this case.
\end{proof}

\section{Directions for Future Work}
An important shortcoming of our approach is that it is propositional so far.
One can make use of ``independently quantified expressions''~\cite{valiant00}
to extend it to some fragments of first-order logic via propositionalization.
But, an interesting question is whether or not we can employ richer 
representations by making use of the symmetries in such instances, along the 
lines of work on relational linear programming by Kersting et al.~\citeyear
{kmt17}.
Finally, we required the implicitly learned constraints to be testable with
probability $1-1/2S$ where $S$ is the na\"{\i}ve norm of the proof. It 
would be desirable to make use of $1-\epsilon$-testable constraints for 
arbitrary $\epsilon$, perhaps only establishing that the portion of
the distribution satisfying the query system has total probability at most 
$\epsilon$.

\bibliographystyle{apalike}
\bibliography{sosinf-starai}
\end{document}